\documentclass{article}
\usepackage{nips12submit_e,times}

\usepackage{amsthm}
\usepackage{amsmath,amssymb}
\usepackage{geometry}
\usepackage{algorithm}
\usepackage{algorithmic}
\usepackage{graphics}
\usepackage{graphicx}
\usepackage{bm}
\usepackage{hyperref}

\newtheorem{definition}{Definition}
\newtheorem{theorem}{Theorem}

\title{Hilbert Space Embedding for Dirichlet Process Mixtures}
\author{
  Krikamol Muandet \\
  Empirical Inference Department \\ 
  Max Planck Institute for Intelligent Systems \\ 
  T\"{u}bingen, Germany \\
  \href{mailto:krikamol@tuebingen.mpg.de}{\texttt{krikamol@tuebingen.mpg.de}}}

\nipsfinalcopy 
 
\begin{document}
\maketitle 
 
\begin{abstract} 
  This paper proposes a Hilbert space embedding for Dirichlet Process mixture models via a stick-breaking construction of Sethuraman \cite{Sethuraman94:Stick}. Although Bayesian nonparametrics offers a powerful approach to construct a prior that avoids the need to specify the model size/complexity explicitly, an exact inference is often intractable. On the other hand, frequentist approaches such as kernel machines, which suffer from the model selection/comparison problems, often benefit from efficient learning algorithms. This paper discusses the possibility to combine the best of both worlds by using the Dirichlet Process mixture model as a case study.
\end{abstract}

\section{Dirichlet Process mixture models}

Much of the real-world data cannot be explained by nice simple probability models. Rather, they often come from heterogeneous sources of unknown properties, which require more complex probability models. Mixture modelling is a popular way of representing such heterogeneity, and also forms a basis for many Bayesian probabilistic models. Unfortunately, a long-standing difficulty in mixture modelling is choosing the number of mixture components, i.e., the number of sources from which the data are generated. Dirichlet Process mixture model (DPMM) allows for the apriori unbounded number of components whose values can be inferred from the observed data.

As a basis of DPMM, we first give a formal definition of the Dirichlet Process (DP), taken from \cite{Ferguson73:DP}.

\begin{definition}[Dirichlet Process]
  \label{def:dp}
  A Dirichlet Process is a distribution of a random probability measure $G$ over a measurable space $(\Omega,\mathcal{B})$, such that for any finite partition $(A_1,\ldots,A_r)$ of $\Omega$ (i.e., $\Omega = \coprod_{i=1}^r A_i$, where $\coprod$ means disjoint union and $A_i\in\mathcal{B}$), we have 
  \begin{equation*}
    (G(A_1),\ldots,G(A_r)) \sim Dir(\alpha G_0(A_1),\ldots,\alpha G_0(A_r))
  \end{equation*}
  \noindent where $G(A_i)=\int_{A_i} dG$ and $G_0(A_i)=\int_{A_i} dG_0$ for $i=1,\ldots,r$.
\end{definition}

Generally speaking, the DP is a distribution over probability measures. Each draw $G$ from a DP can be interpreted as a random distribution, whose sample path is probability measure with probability one. The base distribution $G_0$ can be thought of as the mean of the DP, whereas the strength parameter $\alpha$ can be regarded as an inverse-variance.

The DP has received much attention and has been extensively studied in the past few years, especially in Bayesian nonparametrics community. Several scenarios have been proposed to show the existence of the DP. For example, Blackwell and MacQueen used the Polya urn scheme to show that the distributions sampled from a DP are discrete almost surely \cite{Blackwell73:polya}. Equivalent to the extended Polya urn scheme is a Chinese restaurant process (CRP), which is a random process where $n$ customers sit in a Chinese restaurant with an infinite number of tables. Moreover, one may look at draws from a DP as a weighted sum of point masses. This point was made precise by the stick-breaking construction of Sethuraman \cite{Sethuraman94:Stick}. 

In this paper, we resort to this constructive way of forming $G$. It can be described by the generative process:
\begin{equation*} 
  \beta_i \sim \mathrm{Beta}(1,\alpha), \qquad \pi_i = \beta_i\prod_{k=1}^{i-1}(1-\beta_k), \qquad \theta_i \sim G_0, \qquad G = \sum_{i=1}^{\infty}\pi_i\delta_{\theta_i} \enspace .  
\end{equation*}

The following theorem establishes the connection between the stick-breaking construction and the Dirichlet process given in the Definition \ref{def:dp}.
\begin{theorem}
  The stick-breaking construction gives the same probability measure over all random measures on the measurable space $(\Omega,\mathcal{B})$ with the Dirichlet Process with same parameter $\alpha$ and $G_0$. 
\end{theorem}

By mean of the stick-breaking construction, we consider the Dirichlet Process mixture model (DPMM) of the form $\sum_{i=1}^{\infty}\pi_if_{\theta_i}(x)$, which is a mixture of distributions having the same parametric form $f$ but differing in their parameters. Like many statistical models, exact inference in the DPMM is intractable, and thereby efficient approximate inferences are needed. The most popular inference methods for DPMM are Markov chain Monte Carlo (MCMC), variational Bayesian (VB), and collapsed variational methods. Unlike most previous approaches in nonparametric Bayesian, we study a new approach by employing the Hilbert space embedding. This approach leads to the kernel-based inference for DPMM.

\section{Hilbert space embedding for Dirichlet Process mixtures}

If we consider the base measure $G_0$ to be the distribution over the parameter space $\Theta$ and let $f_{\theta}$, $\theta\sim G_0$ denote the density function parametrized by $\bm{\theta}$. Each draw from the DPMM defines the density function $F_{\bm{\theta}}(x)=\sum_{i=1}^{\infty}\pi_if_{\theta_i}(x)$. We will represent the probability distribution with density $F_{\bm{\theta}}$ as $\mathbb{P}_{\bm{\pi},\bm{\theta}}$ and represent the set of all $\mathbb{P}_{\bm{\pi},\bm{\theta}}$ by $\mathfrak{P}_{\alpha,\Theta}$.

Let $\mathcal{H}$ be the reproducing kernel Hilbert space (RKHS) with a reproducing kernel $k$. Assume that $k(x,x)$ is bounded for all $x$. Then, the Dirichlet Process Mixture Embedding (DPME) is defined as
\begin{equation*}
  \Upsilon \; : \; \mathfrak{P}_{\alpha,\Theta} \longrightarrow \mathcal{H}, \enspace
  \mathbb{P}_{\bm{\pi},\bm{\theta}} \longmapsto \int k(x,\cdot) \,\mathrm{d}\mathbb{P}_{\bm{\pi},\bm{\theta}}(x) \triangleq \sum_{i=1}^{\infty}\pi_i \int k(x,\cdot) \,\mathrm{d}f_{\theta_i}(x)
\end{equation*}

We will denote the embedding of $\mathbb{P}_{\bm{\pi},\bm{\theta}}$ by $\Upsilon[\mathbb{P}_{\bm{\pi},\bm{\theta}}]$. Since we have $\sum_{k=1}^{\infty}\pi_k=1$ almost surely and $k(x,\cdot)<\infty$ for all $x\in\mathcal{X}$, it follows that $\|\Upsilon[\mathbb{P}_{\bm{\pi},\bm{\theta}}]\|^2_{\mathcal{H}}<\infty$. Therefore, the DPME is well-defined.

Unfortunately, working directly with the DPME is cumbersome because of an infinite sum. Ishwaran and James \cite{Ishwaran01:Gibbs} made an important observation that a truncation of the stick-breaking representation at a sufficiently large $T$ already provides an excellent approximation to the full DPMM model. As a result, we propose the \emph{truncated Dirichlet Process Mixture Embedding} (tDPME): 
\begin{equation*}
  \Upsilon \; : \; \mathfrak{P}_{\alpha,\Theta,T} \longrightarrow \mathcal{H}, \enspace
  \mathbb{P}_{\bm{\pi},\bm{\theta},T} \longmapsto \int k(x,\cdot) \,\mathrm{d}\mathbb{P}_{\bm{\pi},\bm{\theta},T}(x) \triangleq \sum_{i=1}^T\pi_i \int k(x,\cdot) \,\mathrm{d}f_{\theta_i}(x)
\end{equation*}
The $\mathfrak{P}_{\alpha,\Theta,T}$ and $\mathbb{P}_{\bm{\pi},\bm{\theta},T}$ denote the truncated version of $\mathfrak{P}_{\alpha,\Theta}$ and $\mathbb{P}_{\bm{\pi},\bm{\theta}}$, respectively, where $T>0$ is a truncation level. The following theorem presents the RKHS version of the almost-sure truncation known in the nonparametric Bayesian literature.

\begin{theorem}[Almost-sure truncation]
  \label{thm:as-truncate}
  Let $\mathcal{H}$ be a reproducing kernel Hilbert space (RKHS) with a reproducing kernel $k$. Assume that $\|k(x,\cdot)\|^2_{\mathcal{H}} \leq R$ for all $x$. The following inequality holds:
  \begin{equation*}
    \left\| \Upsilon[\mathbb{P}_{\bm{\pi},\bm{\theta}}] - \Upsilon[\mathbb{P}_{\bm{\pi},\bm{\theta},T}] \right\|^2_{\mathcal{H}} \leq C\cdot\exp\left(-T/\alpha\right)
  \end{equation*}
  \noindent where $C$ is an arbitrary constant.
\end{theorem}

\begin{proof}
  we have
  \begin{eqnarray*}
    \left\| \Upsilon[\mathbb{P}_{\bm{\pi},\bm{\theta}}] - \Upsilon[\mathbb{P}_{\bm{\pi},\bm{\theta},T}] \right\|^2_{\mathcal{H}} 
    &=& \left\| \sum_{i=1}^{\infty}\pi_i \int k(x,\cdot) \,\mathrm{d}f_{\theta_i}(x) - \sum_{i=1}^T\pi_i \int k(x,\cdot) \,\mathrm{d}f_{\theta_i}(x) \right\|^2_{\mathcal{H}} \\
    &=& \left\| \sum_{i=T+1}^{\infty}\pi_i \int k(x,\cdot) \,\mathrm{d}f_{\theta_i}(x) \right\|^2_{\mathcal{H}} \\
    &\leq& \sum_{i=T+1}^{\infty}\pi_i \int\left\|k(x,\cdot)\right\|^2_{\mathcal{H}} \,\mathrm{d}f_{\theta_i}(x) \leq \sum_{i=T+1}^{\infty}\pi_i \int R \,\mathrm{d}f_{\theta_i}(x) \enspace .
    \end{eqnarray*}
We can see that $\int R \,\mathrm{d}f_{\theta_i}(x)$ is finite for all $i$. Thus, letting $\int R \,\mathrm{d}f_{\theta_i}(x) < C$ for all $i$ with some constant $C$ yields
    \begin{equation*}
    \sum_{i=T+1}^{\infty}\pi_i \int R \,\mathrm{d}f_{\theta_i}(x) \leq \sum_{i=T+1}^{\infty}\pi_i C = C\left(1 - \sum_{i=1}^T\pi_k\right) \approx C\cdot\exp\left(-T/\alpha\right) \enspace .
  \end{equation*}
  The last step of the proof uses the fact that $\sum_{i=1}^T\pi_i = \sum_{i=1}^T\left(\exp(-\Gamma_{i-1}/\alpha)-\exp(-\Gamma_i/\alpha)\right) = 1 - \exp(\Gamma_T/\alpha)\approx 1 - \exp(-T/\alpha)$ where $\Gamma_T=E_1+E_2+\cdots+E_T$ and $E_i\sim\exp(1)$ (cf. \cite{Ishwaran2002}).
\end{proof}

\begin{figure}[t!]
  \centering
  \includegraphics[width=2.7in]{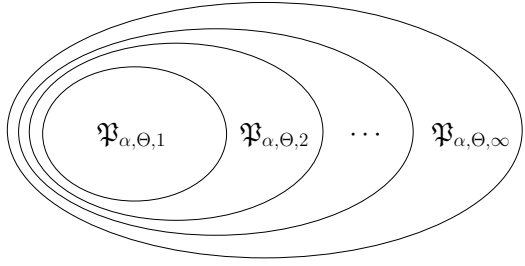}
  \caption{The hierarchical structure of the truncated DPME. The truncation level $T$ imposes a hierarchical structure on the class of distributions $\mathfrak{P}_{\alpha,\Theta,T}$. As $T$ increases, the set $\mathfrak{P}_{\alpha,\Theta,T}$ enlarges, giving more flexibility to the model.}
  \label{fig:hierarchy}
\end{figure}

Theorem \ref{thm:as-truncate} asserts that the truncated DPME is close in RKHS norm to the true DPME with a sufficiently large truncation level $T$. Consequently, by working with the truncated DPME instead of the true DPME, we are not losing much information. Moreover, the bound also suggests how to choose $T$. That is, for the error to be smaller than $\delta$, one must choose $T$ such that $T > \alpha\ln(\delta/C)$. The effect of setting different truncation level $T$ in the DPME can be seen in Figure \ref{fig:hierarchy}. 

\subsection{Optimization}
 
Given observation $x_1,x_2,\ldots,x_m$, we would like to find $\mathbb{P}_{\bm{\pi},\bm{\theta},T}$ that is as close as possible to the underlying distribution $\mathbb{P}$ of the observation. To accomplish this, we employ the usual Hilbert space embedding of $\mathbb{P}$ given by $\mu_{\mathbb{P}}=\mathbb{E}_{x\sim\mathbb{P}}[k(x,\cdot)]$. The empirical estimate of $\mu_{\mathbb{P}}$ can be computed from observation as $\hat{\mu}_X=\frac{1}{m}\sum_{k=1}^m k(x_k,\cdot)$. Then, the optimization problem can be cast as follow:
\begin{align}
  \label{eq:opt1}
  \underset{\bm{\pi}\in\mathbb{R}^{T}}{\min} \quad \|\hat{\mu}_X - \Upsilon[\mathbb{P}_{\bm{\pi},\bm{\theta},T}]\|^2_{\mathcal{H}} && \text{subject to} \quad \bm{\pi}^{\mathsf{T}}\mathbf{1} = 1, \pi_i \geq 0 \enspace .
\end{align}
To prevent overfitting, we introduce a regularizer $\Omega(\bm{\pi})=\frac{1}{2}\|\bm{\pi}\|^2$ with a regularization constant $\varepsilon > 0$. Substituting $\hat{\mu}_X$ and $\Upsilon[\mathbb{P}_{\bm{\pi},\bm{\theta},T}]$ back into \eqref{eq:opt1} yields a quadratic programming (QP) for $\bm{\pi}$:
\begin{align*}
  \underset{\bm{\pi}\in\mathbb{R}^{T}}{\min} \quad \frac{1}{2}\bm{\pi}^{\mathsf{T}}\left(\mathbf{S} + \varepsilon \mathbf{I}\right)\bm{\pi} - \mathbf{R}^{\mathsf{T}}\bm{\pi} && \text{subject to} \quad \bm{\pi}^{\mathsf{T}}\mathbf{1} = 1, \pi_i \geq 0 \enspace ,
\end{align*}
\noindent where $\mathbf{I}$ is the identity matrix, $\mathbf{S}\in\mathbb{R}^{T\times T}$ and $\mathbf{R}\in\mathbb{R}^{T}$ are given by $\mathbf{S}_{ij}=\langle \mu[f_{\theta_i}],\mu[f_{\theta_j}]\rangle_{\mathcal{H}}$ and $\mathbf{R}_j = \langle \hat{\mu}_X,\mu[f_{\theta_j}]\rangle_{\mathcal{H}}$, respectively, and $\mu[f_{\theta_i}] = \int k(x,\cdot) \,\mathrm{d}f_{\theta_i}(x)$. Note that our optimization problem is similar to the one in \cite{Song08:DEKMM}. Thus, due to space constraint, we ask the readers to consult \cite{Song08:DEKMM} on how to compute $\mathbf{S}$ and $\mathbf{R}$ as well as the detail on how to perform an optimization.

The optimization problem we use here is conceptually similar to the variational methods for DPMM \cite{Blei05:VI}. That is, we are minimizing the distance between the approximate quantity $\mathbb{P}_{\bm{\pi},\bm{\theta},T}$ and the true quantity $\mathbb{P}$. Moreover, both MCMC and VB require access to the latent variables associated with observations in order to perform an inference, whereas our approach does not require access to the latent variable whatsoever during an inference. The values of the latent variables, on the other hand, are computed as a postprocessing step.

\section{Discussions}

We are investigating some open questions related to the proposed kernel-based inference of the DPMM. For example, it is vital to understand how the solution of the above optimization problem relates to the solution of the standard inference methods such as maximum likelihood and MAP of the DPMM. Is there a kernel $k$ for which these solutions coincide? What is the effect of choosing different kernel $k$? and what is the connection of our approach to the basic k-mean algorithm? The answers to these questions will be the mutual benefit of researchers in kernel methods and Bayesian nonparametrics.

\footnotesize{
\bibliographystyle{abbrv}
\bibliography{kdpmm}}

\end{document}